\documentclass[runningheads]{llncs}
\pdfoutput=1

\usepackage{amsmath}
\usepackage{amssymb}
\usepackage{makeidx}  
\usepackage{url}
\usepackage[cmtip,arrow]{xy}
\usepackage[curve]{xypic}
\usepackage{graphicx}
\usepackage{stmaryrd}
\usepackage{times}
\usepackage{xparse}
\usepackage{soul}
\usepackage{cancel}

\usepackage[usenames,dvipsnames]{xcolor}

\newcommand{\commentn}[1]{}

\newcommand{\eqdef}{\stackrel{\rm def}{=}}

\def\qed{~\hfill{$\Box$}}
\newtheorem{observation}{Observation}

\newcommand{\tuple}[1]{\ensuremath{\langle #1 \rangle}}
\newcommand{\signature}{\tuple{At, Lb}}
\newcommand{\set}[1]{\ensuremath{\{#1\}}}

\newcommand{\prefix}{\ensuremath{Lb\cup\set{1}}}
\newcommand{\nat}{\mathbb{N}}

\newcommand{\XS}[2]{\ensuremath{\mathcal{X}_{#1}^{#2}}}
\newcommand{\XLb}{\XS{Lb}{}}

\newcommand{\TS}[2]{\ensuremath{\mathcal{T}_{#1}^{#2}}}
\newcommand{\TLb}{\TS{Lb}{}}
\newcommand{\VS}[2]{\ensuremath{\mathcal{V}_{#1}^{#2}}}
\newcommand{\VLb}{\VS{Lb}{}}
\newcommand{\NS}[2]{\ensuremath{\mathcal{U}_{#1}^{#2}}}
\newcommand{\NLb}{\NS{Lb}{}}
\newcommand{\In}{{\ensuremath{\mathcal{I}}}}

\newcommand{\Fi}{\ensuremath{\mathcal{C}}}

\def\bottomI{\ensuremath{\mathbf{0}}}
\def\topI{\ensuremath{\mathbf{1}}}
\newcommand{\oof}[1]{|\!|\!|\!|{#1}|\!|\!|\!|}

\DeclareDocumentCommand\impl{o m m}{\ensuremath{	\IfNoValueTF{#1}{}{{#1}:}{#3}\leftarrow{#2}}}

\newcommand{\tp}[1]{T_\Pi\!\uparrow^{\ #1}(\bottomI)}

\def\Not{\hbox{\em not} \ }

\newcommand{\cT}[2]{PT_{#1}^{#2}}

\newcounter{programcount}
\newcommand{\newprogr}{\refstepcounter{programcount}\ensuremath{\Pi_{\arabic{programcount}}}}
\newcommand{\progrref}[1]{\ensuremath{\Pi_{#1}}}

  \title{An Algebra of Causal Chains\thanks{This research was partially supported by Spanish MEC project TIN2009-14562-C05-04 and Xunta program INCITE 2011.}}
  
  \titlerunning{An Algebra of Causal Chains}

  \author{Pedro Cabalar and Jorge Fandinno}
  
  \authorrunning{P.~Cabalar and J.~Fandi\~no} 
  \institute{	Department of Computer Science\\
				University of Corunna, SPAIN\\
				\email{\{cabalar, jorge.fandino\}@udc.es}
			}

\begin{document}

\maketitle

  \begin{abstract}
  In this work we propose a multi-valued extension of logic programs under the stable models semantics where each true atom in a model is associated with a set of justifications, in a similar spirit than a set of proof trees.
  The main contribution of this paper is that we capture justifications into an algebra of truth values with three internal operations: an addition `$+$' representing  alternative justifications for a formula, a commutative product `$*$' representing joint interaction of causes and a non-commutative product `$\cdot$' acting as a concatenation or proof constructor.
  Using this multi-valued semantics, we obtain a one-to-one correspondence between the syntactic proof tree of a standard (non-causal) logic program and the interpretation of each true atom in a model.
  Furthermore, thanks to this algebraic characterization we can detect semantic properties like redundancy and relevance of the obtained justifications.  
  We also identify a lattice-based characterization of this algebra, defining a direct consequences operator, proving its continuity and that its least fix point can be computed after a finite number of iterations. Finally, we define the concept of \emph{causal stable model} by introducing an analogous transformation to Gelfond and Lifschitz's program reduct.
  \end{abstract}


\section{Introduction}
\label{sec:intro}

A frequent informal way of explaining the effect of default negation in an introductory class on semantics in logic programming (LP) is that a literal of the form `$\Not p$' should be read as ``there is no way to derive $p$.'' Although this idea seems quite intuitive, it is actually using a concept outside the discourse of any of the existing LP semantics: the \emph{ways to derive} $p$. To explore this idea, \cite{Cab12} introduced the so-called \emph{causal logic programs}.
The semantics was an extension of stable models~\cite{GL88} relying on the idea of ``justification'' or ``proof''.
Any true atom, in a standard (non-causal) stable model needs to be justified. In a \emph{causal stable model}, the truth value of each true atom captures these possible justifications, called \emph{causes}.
Let us see an example to illustrate this.

\begin{example}\label{ex:boat}
Suppose we have a row boat with two rowers, one at each side of the boat, port and starboard. The boat moves forward $fwd$ if both rowers strike at a time. On the other hand, if we have a following wind, the boat moves forward anyway.\qed
\end{example}

Suppose now that we have indeed that both rowers stroke at a time when we additionally had a following wind. A possible encoding for this example could be the set of rules \newprogr\label{prg:boat}:
\begin{gather*}
p: port \hspace{10pt} s: starb \hspace{10pt} w: fwind\\
fwd \leftarrow port \wedge starb \hspace{20pt} fwd \leftarrow fwind
\end{gather*}
\noindent 
In the only causal stable model of this program, atom $fwd$ was justified by two alternative and independent causes. On the one hand, cause $\{p,s\}$ representing the joint interaction of $port$ and $starb$. On the other hand, cause $\{w\}$ inherited from $fwind$.
We label rules (in the above program only atoms) that we want to be reflected in causes. Unlabelled $fwd$ rules are just ignored when reflecting causal information.
For instance, if we decide to keep track of the application of these rules, we would handle instead a program \newprogr\label{prg:boat.label} obtained just by labelling these two rules in \progrref{\ref{prg:boat}} as follows:
\begin{align}
a: fwd &\leftarrow port \wedge starb \label{f1}\\
b: fwd &\leftarrow fwind \label{f2}
\end{align}
The two alternative justifications for atom $fwd$ become the pair of causes $\set{p,s} \cdot a$ and $\set{w} \cdot b$. The informal reading of $\set{p,s} \cdot a$ is that ``the joint interaction of $\set{p}$ and $\set{s}$, the cause $\set{p,s}$, is necessary to apply rule $a$.''
From a graphical point of view, we can represent \emph{causes} as proof trees.
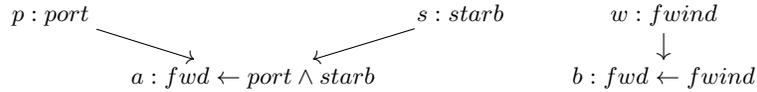
\begin{figure}
\[
\xymatrix @-5mm {
{p : port}	\ar[dr]	&			& {s : starb}	\ar[dl]\\
			& {a : \impl{port \wedge starb}{fwd}} &
}
\hspace{20pt}
\xymatrix @-5mm {
{w:fwind} \ar[d]		\\
{b : \impl{fwind}{fwd}}
}
\]
\caption{Proof trees justifying atom $fwd$ in the program \progrref{\ref{prg:boat.label}}}
\label{fig:tree}
\end{figure}

\noindent In this paper, we show that causes can be embedded in an algebra
with three internal operations: an addition `$+$' representing alternative justifications for a formula, a commutative product `$*$' representing joint interaction of causes
(in a similar spirit to the `$+$' in~\cite{Art01})
and a non-commutative product `$\cdot$' acting as a concatenation or rule application.
Using these operations, we can see that justification for $fwd$ would correspond now to the value
$((p*s) \cdot a) + (w \cdot b)$
which means that $fwd$ is justified by the two alternative \emph{causes}, $(p*s) \cdot a$ and
$(w \cdot b)$. The former refers to the application of rule $a$ to the join interaction of $p$ and $s$. Similarly, the later refers to the application of rule $b$ to $w$.
From a graphical point of view, each cause corresponds to one of proof trees in the Figure~\ref{fig:tree}, the right hand side operator of application corresponds to the head whereas the left hand side operator corresponds to the product of its children.

The rest of the paper is organised as follows. Section~\ref{sec:cterms} describes the algebra with these three operations and a quite natural ordering relation on causes. The next section studies the semantics for positive logic programs and shows the correspondence between the syntactic proof tree of a standard (non-causal) logic program and the interpretation of each atom in a \emph{causal model}. Section~\ref{sec:stable} introduces default negation and stable models. Finally, Section~\ref{sec:conc} concludes the paper.

\section{Algebra of causal values}
\label{sec:cterms}

As we have introduced, our set of \emph{causal values} will constitute an algebra with three internal operations: addition `$+$' representing alternative causes, product `$*$' representing joint interaction between causes and rule application `$\cdot$'. We define now \emph{causal terms}, the syntactic counterpart of (causal) values, just as combinations of these three operations over labels (events).

\begin{definition}[Causal term]
A \emph{causal term}, $t$, over a set of labels $Lb$ is recursively defined as one of the following expressions:
\[
t ::= l \ \ |Ê\ \ \prod_{t_i \in S} t_i \ \ |Ê\ \ \sum_{t_i \in S} t_i \ \ |Ê\ \ t_1 \cdot t_2
\]
\noindent where $l$ is a label $l \in Lb$, $t_1, t_2$ are in their turn causal terms and $S$ is a (possibly empty or possibly infinite) set of causal terms. The set of causal terms over $Lb$ is denoted by $\TLb$.\qed
\end{definition}

As we can see, infinite products and sums are allowed whereas a term may only contain a finite number of concatenation applications.
Constants $0$ and $1$ will be shorthands for the empty sum $\sum_{t \in \emptyset} t$ and the empty product $\prod_{t \in \emptyset}t$, respectively.

We adopt the following notation. To avoid an excessive use of parentheses, we assume that `$\cdot$' has the highest priority, followed by `$*$' and `$+$' as usual, and we further note that the three operations will be associative. When clear from the context, we will sometimes remove `$\cdot$' so that, for instance, the term $l_1 l_2$ stands for $l_1 \cdot l_2$. As we will see, two (syntactically) different causal terms may correspond to the same causal value. However, we will impose Unique Names Assumption (UNA) for labels, that is, $l \neq l'$ for any two (syntactically) different labels $l,l' \in Lb$, and similarly $l \neq 0$ and $l \neq 1$ for any label $l$.

To fix properties of our algebra we recall that addition `$+$' represents a set of alternative causes and product `$*$' a set of causes that are jointly used. Thus, since both represent sets, they are associative, commutative and idempotent.
Contrary, although associative, application `$\cdot$' is not commutative. Note that the right hand side operator represents the applied rule and left hand one represents a cause that is necessary to apply it, therefore they are clearly not interchangeable.
We can note another interesting property: application `$\cdot$' distributes over both addition `$+$' and product `$*$'.
To illustrate this idea, consider the following variation of our example. Suppose now that the boat also leaves a wake behind when it moves forward. Let \newprogr\label{prg:boat.wake} be the set of rules \progrref{\ref{prg:boat}} plus the rule
$k:wake\leftarrow fwd$ reflecting this new assumption.
As we saw, $fwd$ is justified by $p*s+w$ and thus $wake$ will be justified by applying rule $k:wake\leftarrow fwd$ to it, i.e.the value $(p*s+w) \cdot k$. We can also see that there are two alternative causes justifying $wake$, graphically represented in the Figure~\ref{fig:tree.wake}.
\begin{figure}
\[
\xymatrix @-5mm {
{p : port}	\ar[dr]	&			& {s : starb}	\ar[dl]\\
			& {\impl{port \wedge starb}{fwd}} \ar[d]  &\\
			& {k : \impl{fwd}{wake}} &
}
\hspace{20pt}
\xymatrix @-5mm {
{w:fwind} \ar[d]		\\
{\impl{fwind}{fwd}} \ar[d]\\
{k : \impl{fwd}{wake}}
}
\]
\caption{Proof trees pontificating atom $fwd$ in the program \progrref{\ref{prg:boat.wake}}}
\label{fig:tree.wake}
\end{figure}
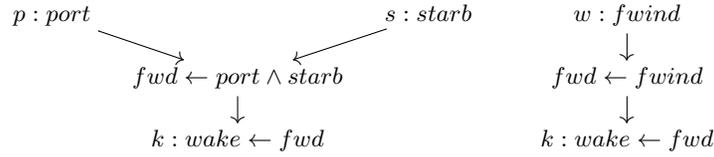
The term that corresponds which this graphical representation is
$(p*s) \cdot k + w \cdot k = (p*s+w) \cdot k$.
Moreover, application `$\cdot$' also distributes over product `$*$' and
$(p*s) \cdot k + w \cdot k$ is equivalent to
$(p \cdot k) * (s \cdot k) + (w \cdot k)$. Intuitively, if the joint iteration of $p$ and $s$ is necessary to apply $k$ then both $p$ and $s$ are also necessary to apply it, and conversely.
Note that each chain of applications , $(p \cdot k)$, $(s \cdot k)$ and $(w \cdot k)$ corresponds to a path in one of the trees in the Figure~\ref{fig:tree.wake}. Causes can be seen as sets (products) of paths (causal chains).

\begin{definition}[Causal Chain]
A \emph{causal chain} $x$ over a set of labels $Lb$ is a sequence $x = l_1 \cdot l_2 \cdot \dotsc \cdot l_n$, or simply $l_1l_2 \dots l_n$, with length $|x| = n>0$ and $l_i \in Lb$. \qed
\end{definition}

We denote \XLb\ to stand for the set of causal chains over $Lb$ and will use letters $x, y, z$ to denote elements from that set. It suffices to have a non-empty set of labels, say $Lb=\set{a}$, to get an infinite set of chains $\XLb=\{a, aa, aaa, \dots\}$, although all of them have a finite length.
It is easy to see that, by an exhaustive application of distributivity, we can ``shift'' inside all occurrences of the application operator so that it only occurs in the scope of other application operators. A causal term obtained in this way is a \emph{normal causal term}.

\begin{definition}[Normal causal term]
A \emph{causal term}, $t$, over a set of labels $Lb$ is recursively defined as one of the following expressions:
\[
t ::= x \mid \prod_{t_i \in S} t_i \mid \sum_{t_i \in S} t_i 
\]
\noindent where $x \in \XLb$ is a causal chain over $Lb$ and $S$ is a (possibly empty or possibly infinite) set of normal causal terms. The set of causal terms over $Lb$ is denoted by $\NLb$.\qed
\end{definition}

\begin{proposition}\label{prof:term.normalize}
Every causal term $t$ can be normalized, i.e. written as an equivalent normal causal term $u$.\qed
\end{proposition}

In the same way as application `$\cdot$' distributes over addition `$+$' and product `$*$', the latter, in their turn, also distributes over addition `$+$'.
Consider a new variation of our example to illustrate this fact.
Suppose that we have now two port rowers that can strike, encoded as the set of rules \newprogr\label{prg:boat.2.rowers}:
\begin{gather*}
p_1: port_1 \hspace{10pt}
p_2: port_2 \hspace{10pt}
s: starb \hspace{10pt}\\
port \leftarrow port_1 \hspace{20pt}
port \leftarrow port_2 \hspace{20pt}
fwd \leftarrow port \wedge starb
\end{gather*}
We can see that, in the only causal stable model of this program, atom $port$ was justified by two alternative, and independent causes, $p_1$ and $p_2$, and after applying unlabelled rules to them, the resulting value  assigned to $fwd$ is  $(p_1+p_2)*s$. It is also clear that there are two alternative causes justifying $fwd$: the result from combining the starboard rower strike with each of the port rower strikes, $p_1*s$ and $p_2*s$. That is, causal terms $(p_1+p_2)*s$ and $p_1*s+p_2*s$ are equivalent.

Furthermore, as we introduce above, causes can be ordered by a notion of ``strength'' of justification. For instance, in our example, $fwd$ is justified by two independent causes, $p*s+w$ while $fwind$ is only justified by $w$.
If we consider the program \newprogr\label{prg:boat.1.cause} obtained by removing the fact $w:fwind$ from \progrref{\ref{prg:boat}} then $fwd$ keeps being justified by $p*s$ but $fwind$ becomes false. That is, $fwd$ is ``more strongly justified'' than $fwind$ in \progrref{\ref{prg:boat}}, written $w \leq p*s+w$. Similarly, $p*s \leq p*s+w$.
Note also that, in this program \progrref{\ref{prg:boat.1.cause}}, $fwd$ needs the joint interaction of $p$ and $s$ to be justified but $port$ and $starb$ only need $p$ and $s$, respectively. That is, $p$ is ``more strongly justified'' than $p*s$, written $p*s \leq p$. Similarly, $p*s \leq s$.
We can also see that in program \progrref{\ref{prg:boat.label}} which labels rules for $fwd$, one of the alternative causes for $fwd$ is $w \cdot b$ and this is ``less strongly justified'' than $w$,
i.e. $w \cdot b \leq w$ since, from a similar reasoning, $w \cdot b$ needs the application of $b$ to $w$ when $w$ only requires itself.
In general, we will see that $a \cdot b \leq a * b \leq X \leq a + b$ where $X$ can be either $a$ or $b$.
We formalize this order relation starting for causal chains. Notice that a causal chain $x = l_1 l_2 \dotsc l_n$ can be alternatively characterized as a partial function from naturals to labels $x : \nat \longrightarrow Lb$ where $x(i) = l_i$ for all $i \leq n$ and undefined for $i>n$. Using this characterisation, we can define the following partial order among causal chains:

\begin{samepage}
\begin{definition}[Chain subsumption]
Given two causal chains $x$ and $y \in \XLb$, we say that $y$ \emph{subsumes} $x$, written $x \leq y$, if and only if there exists a strictly increasing function $\delta: \nat \longrightarrow \nat$ such that for each $i \in \nat$ with $y(i)$ defined,
$x \big( \delta(i) \big) = y(i)$. \qed
\end{definition}
\end{samepage}

\begin{proposition}\label{prop:chain}
Given two finite causal chains $x,y \in \XLb$, they are equivalent (i.e. both $x \leq y$ and $y \leq x$) if and only if they are syntactically identical.\qed
\end{proposition}

\noindent Informally speaking, $y$ subsumes $x$, when we can embed $y$ into $x$, or alternatively when we can form $y$ by removing (or skipping) some labels from $x$.
For instance, take the causal chains $x = a b c d e$ and $y = a c$. Clearly we can form $y=ac= a \cdot \cancel{b} \cdot c \cdot \cancel{d} \cdot \cancel{e}$ by removing $b$, $d$ and $e$ from $x$. Formally, $x \leq y$ because we can take some strictly increasing function with $\delta(1) = 1$ and $\delta(2) = 3$ so that $y(1) = x(\delta(1)) = x(1) = a$ and $y(2) = x(\delta(2)) = x(3) = c$.

Although, at a first sight, it may seem counterintuitive the fact that $x \leq y$ implies $|x|\geq |y|$, as we mentioned, a fact or formula is ``more strongly justified'' when we need to apply less rules to derive it (and so, causal chains contain less labels) respecting their ordering. In this way, chain $a c$ is a ``more strongly justification'' than $a b c d e$.

As we saw above, a cause can be seen as a product of causal chains, that from a graphical point of view correspond to the set of paths in a proof tree.
We notice now an interesting property relating causes and the ``more strongly justified'' order relation:
a joint interaction of comparable causal chains should collapse to the weakest among them. Take, for instance, a set of rules \newprogr\label{prg:prod.red}:
\begin{gather*}
a:p \hspace{15pt}
b: q \leftarrow p \hspace{15pt}
r \leftarrow p \wedge q
\end{gather*}
where, in the unique causal stable model, $r$ corresponds to the value $a * a \cdot b$. Informally we can read this as ``we need $a$ and apply rule $b$ to rule $a$ to prove $r$''. Clearly, we are repeating that we need $a$. Term $a$ is redundant and then
$a * a \cdot b$ is simply equivalent to $a\cdot b$.
This idea is quite related to the definition of \emph{order filter} in order theory. An \emph{order filter} $F$ of a poset $P$ is a special subset $F \subseteq P$ satisfying\footnote{\emph{Order filter} is a weaker notion than \emph{filter} which further satisfies that any pair $x,y \in F$ has a lower bound in $F$ too.} that for any $x \in F$ and $y\in P$, $x \leq y$ implies $y \in F$.
An order filter $F$ is furthermore \emph{generated} by an element $x \in P$ iff $x \leq y$ for all elements $y \in F$, the order filter generated by $x$ is written $||x||$.
Considering causes as the union of filters generated by their causal chains, the join interaction of causes just correspond to their union.
For instance, if we consider the set of labels $Lb = \set{a,b}$ and its corresponding set of causal chains 
$\XLb = \set{a,b,ab,ba,\dotsc}$, then $||ab||$ and $||a||$ respectively correspond to the set of all chains grater than $ab$ and $a$ in the poset $P=\tuple{\XLb,\leq}$. Those are, $||ab||= \set{ab, a, b}$ and $||a||=\set{a}$. The term $a * ab$ corresponds just to the union of both sets $||a|| \cup ||ab|| = ||ab||$.
We define a cause as follows:

\begin{definition}[Cause]
A \emph{cause} for a set of labels $Lb$ is any order filter for the poset of chains $\tuple{\XLb,\leq}$. We will write $\Fi_{Lb}$ (or simply $\Fi$ when there is no ambiguity) to denote the set of all causes for $Lb$.\qed
\end{definition}

This definition captures the notion of cause, or syntactically a product of causal chains. To capture possible alternative causes, that is, additions of products of causal chains, we notice that addition obeys a similar behaviour with respect to redundant causes. Take, for instance, a set of rules \newprogr\label{prg:sum.red}:
\begin{gather*}
a : p	\hspace{15pt}	b : p \leftarrow p
\end{gather*}
It is clear, that the cause $a$ is sufficient to justify $p$, but there are also infinitely many other alternative and redundant causes $a \cdot b$, $a \cdot b \cdot b$, $\dotsc$ that justify $p$, that is $a+a \cdot b + a \cdot b \cdot b + \dotsc$. To capture a set of alternative causes we define the idea of causal value, in its turn, as a filter of causes.
\begin{definition}[Causal Value]
Given a set of labels $Lb$, a \emph{causal value} is any order filter for the poset $\tuple{\Fi_{Lb},\subseteq}$.\qed
\end{definition}
The causal value $\oof{a}$, the filter generated by the cause $||a||$, is the set containing $||a||=\set{a,a+b}$ and all its supersets. That is, $\oof{a}=\set{||a||,||a*b||,||a\cdot b||, \dotsc}$.
Futhermore, as we will se later, addition can be interpreted as the union of causal values for its respective operands.
Thus, $a+a \cdot b + a \cdot b \cdot b + \dotsc$ just corresponds to the union of the causal values generated by their addend causes,
$\oof{a} \cup \oof{a \cdot b} \cup \oof{a \cdot b \cdot b} + \dotsc=\oof{a}$.

The set of possible causal values formed with labels $Lb$ is denoted as $\VLb$. An element from $\VLb$ has the form of a set of sets of causal chains that, intuitively, corresponds to a set of alternative causes (\emph{sum of products of chains}). From a graphical point of view, it corresponds to a set of alternative proof trees represented as their respective sets of paths. We define now the correspondence between syntactical causal terms and their semantic counterpart, causal values.

\begin{definition}[Valuation of normal terms]
The \emph{valuation} of a normal term is a mapping
$\epsilon: \NLb \longrightarrow \VLb$ defined as:
\begin{align*}
\epsilon(x) \eqdef |||x||| \ \text{with } x \in \XLb,
\ \ 
\epsilon\Big( \sum_{ t \in S } t \Big)
	\eqdef \ \ \bigcup_{ t \in S } \epsilon(t),
\ \ 
\epsilon\Big( \prod_{ t \in S } t \Big)
	\eqdef \ \ \bigcap_{ t \in S } \epsilon(t) \hspace{15pt} \Box
\end{align*}
\end{definition}

\noindent Note that any causal term can be normalized and then this definition trivially extends to any causal term. Furthermore, a causal chain $x$ is mapped just to the causal value generated by the cause, in their turn, generated by $x$, i.e. the set containing all causes which contain $x$.
The aggregate union of an empty set of sets (causal values) corresponds to $\emptyset$. Therefore $\epsilon(0) = \bigcup_{ t \in \emptyset } \epsilon(t) = \emptyset$, i.e. $0$ just corresponds to the absence of justification. Similarly, as causal values range over parts of $\Fi$, the aggregate intersection of an empty set of causal values corresponds to $\Fi$, and thus $\epsilon(1) = \bigcap_{ t \in \emptyset } \epsilon(t)= \Fi$, i.e. $1$ just corresponds to the ``maximal'' justification.

\begin{theorem}[From \cite{Stumme97}]\label{theorem:freelattice}
\tuple{\VLb,\cup,\cap} is the free completely distributive lattice generated by
$\tuple{\XLb,\leq}$, and the restriction of $\epsilon$ to \XLb \ is an injective homomorphism (or embedding).\qed
\end{theorem}

The above theorem means that causal terms form a complete lattice. The order relation $\leq$ between causal terms just corresponds to set inclusion between their corresponding causal values, i.e. $x \leq y$ iff $\epsilon(x) \subseteq \epsilon(y)$.
Furthermore, addition `$+$' and product `$*$' just respectively correspond to the least upper bound and the greater lower bound of the associated lattice $\tuple{\TLb,\leq}$ or $\tuple{\TLb,+,*}$ where:
\begin{gather*}
t \leq u  \eqdef   \epsilon(t) \subseteq \epsilon(u)
\hspace{20pt}
(\ \Leftrightarrow \ \ t * u = t \ \ \Leftrightarrow \ \ t + u = u \ )
\end{gather*}
for any normal term $t$ and $u$.
For instance, in our example \progrref{\ref{prg:boat.label}}, $fwd$ was associated with the causal term $p \cdot a * s \cdot a + w \cdot b$. Thus, the causal value associated with it corresponds to
\vspace{-5pt}
$$\epsilon(p \cdot a * s \cdot a + w \cdot b)=
	\oof{p \cdot a} \cap \oof{s \cdot a} \cup \oof{w \cdot b}$$
	
Causal values are, in general, infinite sets. For instance, as we saw before, simply with $Lb=\{a\}$ we have the chains $\XLb=\set{a,aa,aaa,\dotsc}$ and $\epsilon(a)$ contains all possible causes in $\Fi$ that are supersets of $\{a\}$, that is, $\epsilon(a)=\{\{a\},\{aa,a\}, \{aaa,aa,a\},\dots\}$. Obviously, writing causal values in this way is infeasible -- it is more convenient to use a representative causal term instead. For this purpose, we define a function $\gamma$ that acts as a right inverse morphism for $\epsilon$ selecting minimal causes, i.e., given a causal value $V$, it defines a normal term $\gamma(V)=t$ such that $\epsilon(t)=V$ and $\gamma(V)$ does not have redundant subterms. The function $\gamma$ is defined as a mapping $\gamma : \VLb \longrightarrow \NLb$ 
such that for any causal value $V \in \VLb$, $\gamma(V) \eqdef \sum_{C \in \underline{V}} \prod_{x \in \underline{C}} x$ where $\underline{V} = \set{ C \in V \mid \not\exists D \in V, D \subset C}$ and
$\underline{C} = \set{ x \in C \mid \not\exists y \in C, y < x }$ respectively stand for $\subseteq$-minimal causes of $V$ and $\leq$-minimal chains of $C$. We will use $\gamma(V)$ to represent $V$.

\begin{proposition}\label{prop:freelattice.map}
The mapping $\gamma$ is a right inverse morphism of $\epsilon$.\qed
\end{proposition}

Given a term $t$ we define its \emph{canonical form} as $\gamma(\epsilon(t))$. Canonical terms are of the form of sums of products of causal chains.
As it can be imagined, not any term in that form is a canonical term.
For instance, going back, we easily can check that terms $a*ab=ab$ and $a+ab+abb+\cdots=a$ respectively correspond to the canonical terms $\gamma(\epsilon(ab*a))=\gamma(\epsilon(ab))=ab$ and
$\gamma(\epsilon(a+ab+abb+\dotsc))=\gamma(\epsilon(a))=a$.
Figure~\ref{fig:DBLattice} summarizes addition and product properties while Figure~\ref{fig:appl} is analogous for application properties.

\begin{figure}[htbp]
\footnotesize
\begin{center}
$
\begin{array}{c}
\hbox{\em Associativity} \\
\hline
\begin{array}{r@{\ }c@{\ }r@{}c@{}l c r@{}c@{}l@{\ }c@{\ }l@{\ }}
t & + & (u & + & w) & = & (t & + & u) & + & w\\
t & * & (u & * & w) & = & (t & * & u) & * & w
\end{array}
\end{array}
$
\ \
$
\begin{array}{c}
\ \ \ \ \hbox{\em Commutativity}\ \ \ \ \\
\hline
\begin{array}{r@{\ }c@{\ }l c r@{\ }c@{\ }l@{\ }}
t & + & u & = & u & + & t\\ 
t & * & u & = & u & * & t
\end{array}
\end{array}
$
\ \
$
\begin{array}{c}
\hbox{\em Absorption} \\
\hline
\begin{array}{c c r@{\ }c@{\ }r@{}c@{}l@{\ }}
t & = & t & + & (t & * & u)\\
t & = & t & * & (t & + & u)
\end{array}
\end{array}
$
\ \
\\
\vspace{10pt}
$
\begin{array}{c}
\hbox{\em Distributive} \\
\hline
\begin{array}{r@{\ }c@{\ }r@{}c@{}l c r@{}c@{}l@{\ }c@{\ }r@{}c@{}l@{}}
t & + & (u & * & w) & = & (t & + & u) & * & (t & + & w)\\
t & * & (u & + & w) & = & (t & * & u) & + & (t & * & w)
\end{array}
\end{array}
$
\ \
$
\begin{array}{c}
Identity \\
\hline
\begin{array}{rcr@{\ }c@{\ }l@{\ }}
t & = & t & + & 0\\
t & = & t & * & 1
\end{array}
\end{array}
$
\ \
$
\begin{array}{c}
\hbox{\em Idempotence} \\
\hline
\begin{array}{rcr@{\ }c@{\ }l@{\ }}
t & = & t & + & t\\
t & = & t & * & t
\end{array}
\end{array}
$
\ \ 
$
\begin{array}{c}
\hbox{\em Annihilator} \\
\hline
\begin{array}{rcr@{\ }c@{\ }l@{\ }}
1 & = & 1 & + & t\\
0 & = & 0 & * & t
\end{array}
\end{array}
$
\end{center}
\caption{Sum and product satisfy the properties of a completely distributive lattice.}
\label{fig:DBLattice}
\end{figure}

\begin{figure}[htbp]
\begin{center}
$
\begin{array}{c}
\hbox{\em Associativity}\\
\hline
\begin{array}{r@{\ }c@{\ }r@{}c@{}l c r@{}c@{}l@{\ }c@{\ }l@{\ }}
t & \cdot & (u & \cdot & w) & = & (t & \cdot & u) & \cdot & w\\
\\
\end{array}
\end{array}
$
\ \
$
\begin{array}{c}
\hbox{\em Absorption}\\
\hline
\begin{array}{r@{\ }c@{\ }c@{\ }c@{\ }l c r@{\ }c@{\ }r@{\ }c@{\ }c@{\ }c@{\ }c@{\ }l@{\ }}
&& t &&& = & t & + & u & \cdot & t & \cdot & w \\
u & \cdot & t & \cdot & w & = & t & * & u & \cdot & t & \cdot & w
\end{array}
\end{array}
$
\ \
$
\begin{array}{c}
\hbox{\em Identity}\\
\hline
\begin{array}{rc r@{\ }c@{\ }l@{\ }}
t & = & 1 & \cdot & t\\
t & = & t & \cdot & 1
\end{array}
\end{array}
$
\\
\vspace{10pt}
$
\begin{array}{c}
\hbox{\em Addition\ distributivity}\\
\hline
\begin{array}{r@{\ }c@{\ }r@{}c@{}l c r@{}c@{}l@{\ }c@{\ }r@{}c@{}l@{}}
t & \cdot & (u & + & w) & = & (t & \cdot & u) & + & (t & \cdot & w)\\
( t & + & u ) & \cdot & w & = & (t & \cdot & w) & + & (u & \cdot & w)\\
\end{array}
\end{array}
$
\ \ 
$
\begin{array}{c}
\hbox{\em Product\ distributivity}\\
\hline
\begin{array}{r@{\ }c@{\ }l@{\ }c@{\ }l c r@{\ }c@{\ }l@{\ }c@{\ }r@{\ }c@{\ }l@{\ }}
c & \cdot & (d & * & e) & = & (c & \cdot & d) & * & (c & \cdot & e)\\
( c & * & d ) & \cdot & e & = & (c & \cdot & e) & * & (c & \cdot & e)\\
\end{array}
\end{array}
$
\ \
$
\begin{array}{c}
\hbox{\em Annihilator}\\
\hline
\begin{array}{rc r@{\ }c@{\ }l@{\ }}
0 & = & t & \cdot & 0\\
0 & = & 0 & \cdot & t\\
\end{array}
\end{array}
$
\end{center}
\caption{Properties of the application `$\cdot$' operator. Note: $c$, $d$ and $e$ denote a causes instead of arbitrary causal terms.}
\label{fig:appl}
\end{figure}

For practical purposes, simplification of causal terms can be done by applying the algebraic properties shown in Figures~\ref{fig:DBLattice}~and~\ref{fig:appl}.
For instance, the examples from \progrref{\ref{prg:prod.red}} and \progrref{\ref{prg:sum.red}} containing redundant information can now be derived as follows:
\[
\begin{array}{rcl@{\hspace{15pt}}l}
a*a\cdot b & = &  (a*1 \cdot a \cdot b) & \text{identity for } \hbox{`$\cdot$'}\\
& = &  1 \cdot a \cdot b & \text{absorption for } \hbox{`$\cdot$'}\\
& = &  a  \cdot b & \text{identity for } \hbox{`$\cdot$'}\\
\\
a+a \cdot b+a \cdot b \cdot b+\dotsc
	& = &  a+ 1 \cdot a \cdot b + a \cdot b \cdot b + \dotsc & \text{identity for } \hbox{`$\cdot$'}\\
& = &  a + a \cdot b \cdot b + \dotsc
	& \text{absorption for } \hbox{`$\cdot$'}\\
& = &  a + 1 \cdot a \cdot b \cdot b + \dotsc & \text{identity for } \hbox{`$\cdot$'}\\
& \dotsc & & \dotsc\\
& = &  a & \text{absorption for } \hbox{`$\cdot$'}
\end{array}
\]

\noindent Let us see another example involving distributivity. The term $ab*c +a$ can be derived as follows:
\[
\begin{array}{rcl@{\hspace{15pt}}l}
a\cdot b*c+a & = & (a\cdot b+a) * (c+a) & \text{distributivity}\\
& = &  (1 \cdot a \cdot b+a) * (c+a) & \text{identity for } \hbox{`$\cdot$'}\\
& = &  (a+1 \cdot a \cdot b) * (c+a) & \text{commutativity for } \hbox{`$+$'}\\
& = &  a * (c+a) & \text{absorption for } \hbox{`$\cdot$'}\\
& = &  a & \text{absorption for } \hbox{`$*$'}\\
\end{array}
\]

\section{Positive programs and minimal models}

Let us describe now how to use the causal algebra to evaluate causal logic programs.
A \emph{signature} is a pair \signature\ of sets that respectively represent the set of \emph{atoms} (or \emph{propositions}) and the set of labels. 
As usual, a \emph{literal} is defined as an atom $p$ (positive literal) or its negation $\neg p$ (negative literal). In this paper, we will concentrate on programs without disjunction in the head, leaving the treatment of disjunction for a future study. 

\begin{definition}[Causal logic program]\label{def:causal.P}
Given a signature $\langle At,Lb\rangle$ a \emph{(causal) logic program} $\Pi$ is a set of rules of the form:
$$t: L_0 \leftarrow L_1 \wedge \dotsc \wedge L_m \wedge \Not L_{m+1} \wedge \dotsc \wedge \Not L_n$$
where $t$ is a causal term over $Lb$, $L_0$ is a literal or $\bot$ (the \emph{head} of the rule) and
$L_1 \wedge \dotsc \wedge L_m \wedge \Not L_{m+1} \wedge \dotsc \Not L_n$ is a conjunction of literals (the \emph{body} of the rule). An empty body is represented as $\top$.\qed
\end{definition}

For any rule $\phi$ of the form
$t: L_0 \leftarrow L_1 \wedge \dotsc \wedge L_m \wedge \Not L_{m+1} \wedge \dotsc \Not L_n$ we define $label(\phi)=t$. Most of the following definitions are standard in logic programming. We denote $head(\phi)=L_0$, $B^+$ (resp. $B^-$) to represent the conjunction of all positive (resp. negative) literals $L_1 \wedge \dotsc \wedge L_n$ (resp. $\Not L_{m+1} \wedge \dotsc \wedge \Not L_n$) that occur in $B$. A logic program is \emph{positive} if $B^-$ is empty for all rules ($n=m$), that is, if it contains no negations.
Unlabelled rules are assumed to be labelled with the element $1$ which, as we saw, is the identity for application `$\cdot$'. $\top$ (resp. $\bot$) represent truth (resp. falsity). If $n=m=0$ then $\leftarrow$ can be dropped.

Given a signature \signature\ a \emph{causal interpretation} is a mapping $I:At~\longrightarrow~\mathcal{V}_{Lb}$ assigning a causal value to each atom. Partial order $\leq$ is extended over interpretations so that given two interpretations $I,J$ we define $I\leq J \eqdef I(p) \leq J(p)$ for each atom $p \in At$. There is a $\leq$-bottom interpretation \bottomI\ (resp. a $\leq$-top interpretation \topI) that stands for the interpretation mapping each atom $p$ to $0$ (resp. $1$). The set of interpretations $\In$ with the partial order $\leq$ forms a poset $\tuple{\In,\leq}$ with supremum `$+$' and infimum `$*$' that are respectively the sum and product of atom interpretations. As a result, \tuple{\In,+,*}\ also forms a complete lattice.

\begin{observation}\label{obs:classic}
When $Lb=\emptyset$ the set of causal values becomes $\VLb=\{0,1\}$ and interpretations collapse to classical propositional logic interpretations.\qed
\end{observation}

\begin{definition}[Causal model]\label{def:causal.M}
Given a positive causal logic program $\Pi$ and a causal interpretation $I$ over the signature \signature, $I$ is a causal model, written $I \models \Pi$, if and only if
\vspace{-10pt}
\begin{align*}
\big( I(L_1) * \dotsc * I(L_m) \big) \cdot t \leq I(L_0)
\end{align*}
for each rule $\varphi \in \Pi$ of the form
$\varphi = L_0 \leftarrow L_1, \dotsc, L_m$.
\end{definition}

\noindent For instance, take rule \eqref{f1} from Example~\ref{ex:boat} and let $I$ be an interpretation such that $I(port)=p$ and $I(starb)=s$.
Then $I$ will be a model of \eqref{f1} when $(p*s)\cdot a \leq I(fwd)$. In particular, this holds when $I(fwd)=(p*s)\cdot a+w\cdot b$ which was the value we expected for program $\Pi_2$. But it would also hold when, for instance, $I(fwd)=a+b$ or $I(fwd)=1$. Note that this is important if we had to accommodate other possible additional facts $(a: fwd)$ or even $(1:fwd)$ in the program. The fact that any $I(fwd)$ greater than $(p*s)\cdot a+w\cdot b$ is also a model clearly points out the need for selecting minimal models. In fact, as happens in the case of non-causal programs, positive programs have a least model (this time, with respect to $\leq$ relation among causal interpretations) that can be computed by iterating an extension of the well-known \emph{direct consequences operator} defined by~\cite{vEK76}.

\begin{definition}[Direct consequences]\label{def:tp}
Given a positive logic program $\Pi$ for signature $\tuple{At,Lb}$ and a causal interpretation $I$, the operator of \emph{direct consequences} is a function $T_\Pi : \In  \longrightarrow \In$ such that, for any atom $p \in At$:
\begin{align*}
T_\Pi(I)(L_0)
	\eqdef \sum \big\{ \ \big( I(L_1) * \dotsc * I(L_m) \big) \cdot t \ \mid \ (t: L_0 \leftarrow L_1 \wedge \dotsc \wedge L_m ) \in \Pi \ \big\}
\end{align*}
\end{definition}

In order to prove some properties of this operator, an important observation should be made: since the set of causal values forms now a lattice, causal logic programs can be translated to
\emph{Generalized Annotated Logic Programming} (GAP).  GAP
is a general a framework for multivalued logic programming where the set of truth values must to form an upper semilattice and rules (\emph{annotated clauses}) have  the following form:
\begin{align}
L_0 : \rho
	\leftarrow L_1 : \mu_1 \ \&\ \dotsc \ \&\ L_m : \mu_m
	\label{eq:gap.rule}
\end{align}
where $L_0, \dotsc, L_m$ are literals, $\rho$ is an \emph{annotation} (may be just a truth value, an \emph{annotation variable} or a \emph{complex annotation}) and $\mu_1, \dotsc,\mu_m$ are values or annotation variables. A complex annotation is the result to apply a total continuous function to a tuple of annotations. For instance $\rho$ can be a complex annotation $f(\mu_1,\dotsc,\mu_m)$ that applies the function $f$ to a m-tuple $(\mu_1,\dotsc,\mu_m)$ of annotation variables in the body of \eqref{eq:gap.rule}.
Given a positive program $\Pi$, each rule $\varphi \in \Pi$ of the form
\begin{align}
t: L_0 \leftarrow L_1 \wedge \dotsc \wedge L_m 
	\label{eq:pos.rule}
\end{align}
is translated to an annotated clause $GAP(\varphi)$ of the form of \eqref{eq:gap.rule} where $\mu_1, \dotsc,\mu_m$ are annotation variables that capture the causal values of each body literal.
The head complex annotation corresponds to the function
$\rho \eqdef (\mu_1 * \dotsc * \mu_m) \cdot t$ that forces the head to inherit the causal value obtained by applying the rule label $t$ to the product of the interpretation of body literals $\mu_1 * \dotsc * \mu_m$.
The translation of a program $\Pi$ is simply defined as: 
\begin{eqnarray*}
 GAP(\Pi) & \eqdef & \{ GAP(\varphi) \mid \varphi \in \Pi \}
\end{eqnarray*}
\noindent A complete description of GAP restricted semantics, denoted as $\models^r$, is out of the scope of this paper (the reader is referred to~\cite{KiferS92}). For our purposes, it suffices to observe that the following important property is satisfied.

\begin{theorem}\label{theorem:gap.translation}
A positive causal logic program $\Pi$ can be translated to a general annotated logic program $GAP(\Pi)$ s.t. a causal interpretation $I \models \Pi$ if and only if $I \models^r GAP(\Pi)$. Furthermore, $T_\Pi(I)=R_{GAP(\Pi)}(I)$ for any interpretation $I$.
\end{theorem}

\begin{samepage}
\begin{corollary}\label{theorem:tp.properties}
Given a positive logic program $\Pi$ the following properties hold:
\begin{enumerate}
\item Operator $T_\Pi$ is monotonic.
\item Operator $T_\Pi$ is continuous.
\item $\tp{\omega} = lfp(T_\Pi)$ is the least model of $\Pi$.
\item The iterative computation $T_\Pi\!\uparrow^{\ k}(\bottomI)$ reaches the least fixpoint in $n$ steps for some positive integer $n$.
\end{enumerate}
\end{corollary}
\begin{proof}
Directly follows from Theorem~\ref{theorem:gap.translation} and Theorems~1,~2~and~3 in \cite{KiferS92}. 
\end{proof}
\end{samepage}

The existence of a least model for a positive program and its computation using $T_\Pi$ is an interesting result, but it does not provide any information on the relation between the causal value it assigns to each atom with respect to its role in the program. As we will see, we can establish a direct relation between this causal value and the idea of \emph{proof} in the positive program.
Let us formalise next the idea of proof tree.

\begin{definition}[Proof tree]
Given a causal logic program $\Pi$, a \emph{proof tree} is a directed acyclic graph $T = \tuple{V,E}$, where vertices $V \subseteq \Pi$ are rules from the program, and $E \subseteq V \times V$ satisfying:
\begin{itemize}
\item[(i)] There is at exactly one vertex without outgoing edges denoted as $sink(T)$.
\item[(ii)] For each rule $\varphi=(t : L_0 \leftarrow B ) \in V$ and for each atom $L_i \in B^+$ there is exactly one $\varphi'$ with $(\varphi',\varphi) \in E$ and this rule satisfies $head(\varphi')=L_i$.\qed
\end{itemize}
\end{definition}
\noindent Notice that condition (ii) forces us to include an incoming edge for each atom in the positive body of a vertex rule. As a result, source vertices must be rules with empty positive body, or just facts in the case of positive programs. Another interesting observation is that, proof \emph{tree} do not require an unique{} parent for each vertex. For instance, in Example~\ref{ex:boat}, if both $port$ and $starb$ were obtained as a consequence of some command made by the captain, we could get instead a proof tree, call it $T_1$, of the form:
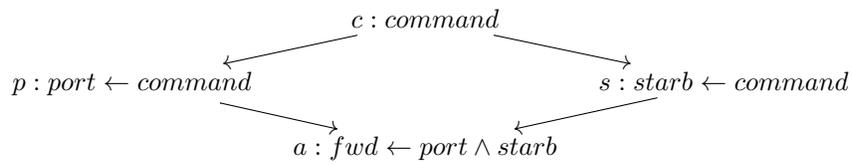
\[
\xymatrix @-5mm {
& {c: command} \ar[dl] \ar[dr] &\\
{p : \impl{command}{port}} \ar[dr] & & {s : \impl{command}{starb}} \ar[dl] \\
& {a : \impl{port \wedge starb}{fwd}} &
}
\]
\begin{definition}[Proof path]
Given a proof tree $T=\tuple{V,E}$ we define a \emph{proof path} for $T$ as a concatenation of terms $t_1 \dots t_n$ satisfying:
\begin{enumerate}
\item There exists a rule $\varphi \in V$ with $label(r)=t_1$ such that $\varphi$ is a source, that is, there is no $\varphi'$ s.t. $(\varphi',\varphi) \in E$.
\item For each pair of consecutive terms $t_i, t_{i+1}$ in the sequence, there is some edge $(\varphi_{i}, \varphi_{i+1}) \in E$ s.t. $label(\varphi_i)=t_i$ and $label(\varphi_{i+1})=t_{i+1}$.
\item $label(sink(T))=t_n$.\qed
\end{enumerate}
\end{definition}

Let us write $Paths(T)$ to stand for the set of all proof paths for a given proof tree $T$. We define the cause associated to any tree $T=\tuple{V,E}$ as the causal term $cause(T) \eqdef \prod_{t \in Paths(T)} t$. As an example, $cause(T_1)= (c \cdot p \cdot a) * (c \cdot s \cdot a)$. Also $(p \cdot a) * (s \cdot a)$ and $w \cdot b$ correspond to each tree in Figure~\ref{fig:tree}.

\begin{theorem}\label{theorem:horn.tree.least}
Let $\Pi$ be a positive program and $I$ be the least model of $\Pi$, then for each atom $p$:
$$I(p) = \sum_{T \in \cT{p}{}} cause(T)$$ where 
$\cT{p}{} = \set{ T = \tuple{V,E} \mid head(sink(T))=p}$ is a set of proof trees with nodes $V \subseteq \Pi$.
\end{theorem}

\noindent From this result, it may seem that our semantics is just a direct translation of the syntactic idea of proof trees. However, the semantics is actually a more powerful notion that allows detecting redundancies, tautologies and inconsistencies. In fact, the expression $\sum_{T \in \cT{p}{}} cause(T)$ may contain redundancies and is not, in the general case, in normal form. As an example, recall the program \progrref{\ref{prg:prod.red}}:
\begin{gather*}
a:p \hspace{15pt}
b: q \leftarrow p \hspace{15pt}
r \leftarrow p \wedge q
\end{gather*}
\noindent that has only one proof tree for $p$ whose cause would correspond to $I(r)=a*a \cdot b$. But, by absorption, this is equivalent to $I(r)= a \cdot b$ pointing out that the presence of $p$ in rule $r \leftarrow p \wedge q$ is redundant.

A corollary of Theorem~\ref{theorem:horn.tree.least} is that we can replace a rule label by a different one, or by $1$ (the identity for application `$\cdot$') and we get the same least model, modulo the same replacement in the causal values for all atoms.
\begin{corollary}\label{corollary:horn.min.reduction}
Let $\Pi$ be a positive program, $I$ the least model of $\Pi$, $l \in Lb$ be a label, $m \in \prefix$ and $\Pi^l_{m}$ (resp. $I^l_{m}$) be the program (resp. interpretation) obtained after replacing each occurrence of $l$ by $m$ in $\Pi$ (resp. in the interpretation of each atom in $I$). Then $I^l_{m}$ is the least model of $\Pi^l_{m}$.\qed
\end{corollary}

\noindent In particular, replacing a label by $m = 1$ has the effect of removing it from the signature. Suppose we make this replacement for all atoms in $Lb$ and call the resulting program and least model $\Pi^{Lb}_1$ and $I^{Lb}_1$ respectively. Then $\Pi^{Lb}_{1}$ is just the non-causal program resulting from $\Pi$ after removing all labels and it is easy to see (Observation~\ref{obs:classic}) that $I^{Lb}_1$ coincides with the least classical model of this program\footnote{Note that $I^{Lb}$ is Boolean: if assigns either $0$ or $1$ to any atom in the signature.}. Moreover, this means that for any positive program $\Pi$, if $I$ is its least model, then the classical interpretation:
\begin{eqnarray*}
 I'(p) &\eqdef & \begin{cases}
  1 &\text{if } I(p) \neq 0 \\
  0 &\text{otherwise}
\end{cases}
\end{eqnarray*}
\noindent is the least classical model of $\Pi$ ignoring its labels.

\section{Default negation and stable models}
\label{sec:stable}

Consider now the addition of negation, so that we deal with arbitrary programs.
To illustrate this, we introduce a variation in Example~\ref{ex:boat} introducing the \emph{qualification problem} from \cite{McCarthy77}:  actions for moving the boat forward can be disqualified if an \emph{abnormal situation} occurs (for instance, that the boat is anchored, any of the oars are broken, the sail is full of holes, etc.) . As usual this can be represented using default negation as shown in the set of rules \newprogr\label{prg:boat.anchor}:
\begin{gather*}
	\begin{gathered}
p: port \hspace{20pt}	s: starb
	\\
	a: fwd \leftarrow port \wedge starb \wedge \Not ab\_a
	\\
	ab\_a \leftarrow anchored
	\\
	ab\_a \leftarrow broken\_oar1
	\\
	ab\_a \leftarrow broken\_oar2
	\end{gathered}
	\hspace{20pt}
	\begin{gathered}
	w:fwind
	\\
	b: fwd \leftarrow fwind \wedge \Not ab\_b
	\\
	ab\_b \leftarrow anchored
	\\
	ab\_b \leftarrow holed\_sail
	\\
	\dotsc
	\end{gathered}
\end{gather*}%

\noindent The causes that justify an atom should not be a list of not occurred abnormal situations. For instance, in program \progrref{\ref{prg:boat.anchor}} where no abnormal situation occurs, the causal value that justify atom $fwd$ should be $(p \cdot a*s \cdot a)+(w \cdot b)$ as in the program \progrref{\ref{prg:boat.label}} where abnormal situations are not included. References to the not occurred abnormal situations
($\Not anchored$, $\Not broken\_oar1$\dots) are not mentioned. Default negation does not affect the causes justifying an atom when the default holds. Of course, when the default does not hold, for instance adding the fact $anchored$ to the above program, $fwd$ becomes false.
Thus, we introduce the following straightforward rephrasing of the
traditional program reduct \cite{GL88}.

\begin{samepage}
\begin{definition}[Program reduct]
The reduct of a program $\Pi$ with respect to
an interpretation $I$, written $\Pi^I$ is the result of the following transformations on $\Pi$:
\begin{enumerate}
\item Removing all rules s.t. $I(B^-) = 0$
\item Removing all negative literals from the rest of rules.\qed
\end{enumerate}
\end{definition}
\end{samepage}

A causal interpretation $I$ is a \emph{causal stable model} of a causal program $\Pi$ if $I$ is the least model of $\Pi^I$. This definition allows us to extend Theorem~\ref{theorem:horn.tree.least} to normal programs in a direct way:

\begin{theorem}[Main theorem]\label{theorem:normal.tree.least}
Let $\Pi$ be a causal program and $I$ be causal stable model of $\Pi$, then for each atom $p$:
$$I(p) = \sum_{T \in \cT{p}{}} cause(T) \hspace{30pt} \hbox{where}$$
$\cT{p}{} = \set{ T = \tuple{V,E} \mid head(sink(T))=p \ \hbox{and } V \subseteq \set{ (t: \impl{B}{q}) \in \Pi \mid I(B^-)\neq 0 }}$.\qed
\end{theorem}
\noindent That is, the only difference now is that the set of proof trees $\cT{p}{}$ is formed with rules whose negative body is not false $I(B^-) \neq 0$ (that is, they would generate rules in the reduct).

\begin{corollary}\label{corollary:normal.min.reduction}
Let $\Pi$ be a normal program, $I$ a causal stable model of $\Pi$, $l \in Lb$ be a label, $m \in \prefix$ and $\Pi^l_{m}$ (resp. $I^l_{m}$) be the program (resp. interpretation) obtained after replacing every occurrence of $l$ by $m$ in $\Pi$ (resp. in the interpretation of each atom in $I$). Then $I^l_{m}$ is a causal stable model of $\Pi^l_{m}$.\qed
\end{corollary}

As in the case of positive programs, replacing a label by $m = 1$ has the effect of removing it from the signature. Then, for any normal program $\Pi$, if $I$ is a causal stable model, then the classical interpretation:
\begin{eqnarray*}
 I'(p) &\eqdef & \begin{cases}
  1 &\text{if } I(p) \neq 0 \\
  0 &\text{otherwise}
\end{cases}
\end{eqnarray*}
is a classical stable model of $\Pi$ ignoring its labels. It is easy to see that not only the above program \progrref{\ref{prg:boat.anchor}} has an unique causal stable model that corresponds to:
\begin{gather*}
\begin{array}{lcl}
I(port) &=& p
\\
I(starb)&=& s
\\
I(fwind)&=& w
\\
I(fwd)&=& (p \cdot a*s \cdot a)+(w \cdot b)
\end{array}
\hspace{1cm}
\begin{array}{lcl}
I(ab\_f)&=& 0
\\
I(anchored)&=& 0
\\
I(broken\_oar1) &=& 0
\\
\hspace{0.75cm}\dotsc &=& 0
\end{array}
\end{gather*}
but also the program obtained from it ignoring the labels has an unique standard stable model $\set{port, starb, fwind, fwd}$ that corresponds to the atoms whose interpretations differ from $0$ in the former.

\section{Conclusions}
\label{sec:conc}

In this paper we have provided a multi-valued semantics for normal logic programs whose truth values form a lattice of causal chains. A causal chain is nothing else but a concatenation of rule labels that reflects some sequence of rule applications. In this way, a model assigns to each true atom a value that contains justifications for its derivation from existing rules. We have further provided three basic operations on the lattice: an addition, that stands for alternative, independent justifications; a product, that represents joint interaction of causes; and a concatenation that acts as a chain constructor. We have shown that this lattice is completely distributive and provided a detailed description of the algebraic properties of its three operations.

A first important result is that, for positive programs, there exists a least model that coincides with the least fixpoint of a direct consequences operator, analogous to~\cite{vEK76}. With this, we are able to prove a direct correspondence between the semantic values we obtain and the syntactic idea of proof tree. The main result of the paper generalises this correspondence for the case of stable models for normal programs.

Many open topics remain for future study. For instance, ongoing work is currently focused on implementation, complexity assessment, extension to disjunctive programs or introduction of strong negation. Regarding expressivity, an interesting topic is the introduction of new syntactic operators for inspecting causal information like checking the influence of a particular event or label in a conclusion, expressing necessary or sufficient causes, or even dealing with counterfactuals. Another interesting topic is removing the syntactic reduct definition in favour of some full logical treatment of default negation, as happens for (non-causal) stable models and their characterisation in terms of Equilibrium Logic \cite{Pearce06}. This would surely simplify the quest for a necessary and sufficient condition for strong equivalence, following similar steps to~\cite{LPV01}. It may also allow extending the definition of causal stable models to an arbitrary syntax and to the first order case, where the use of variables in labels may also introduce new interesting features. 

There are also other areas whose relations deserve to be formally studied. For instance, the introduction of a strong negation operator will immediate lead to a connection to Paraconsistency approaches. In particular, one of the main problems in the area of Paraconsistency is deciding which parts of the theory do not propagate or \emph{depend} on an inconsistency. This decision, we hope, will be easier in the presence of causal justifications for each derived conclusion. A related area for which similar connections can be exploited is Belief Revision. In this case, causal information can help to decide which \emph{relevant} part of a revised theory must be withdrawn in the presence of new information that would lead to an inconsistency if no changes are made. A third obvious related area is Debugging in Answer Set Programming, where we try to explain discrepancies between an expected result and the obtained stable models. In this field, there exists a pair of relevant approaches~\cite{Gebser08,Pontelli09} to whom we plan to compare. Finally, as potential applications, our main concern is designing a high level action language on top of causal logic programs with the purpose of modelling some typical scenarios from the literature on causality in Artificial Intelligence.

%
\bibliographystyle{splncs}
\bibliography{refs}

\end{document}